\documentclass[12pt]{article}
\usepackage{graphicx}
\usepackage{makeidx}
\usepackage{amsmath}
\usepackage{amsfonts}
\usepackage{amssymb}
\long\def\remove#1{}

\newcommand{\AK}{\mbox{${\cal{A}}_K$ }}

\newcommand{\pbox}{\hbox to 6pt{\leaders\hrule width 6pt height 6pt\hfill}}

\newtheorem{definition}{Definition}[section]
\newtheorem{theorem}{Theorem}[section]
\newtheorem{lemma}{Lemma}[section]

\newtheorem{proposition}{Proposition}[section]
\newtheorem{example}{Example}[section]
\newtheorem{remark}{Remark}[section]

\newenvironment{proof}{\parindent=0pt{\bf Proof: }}{
   \hspace*{\fill}\hbox to 6pt{\leaders\hrule width 6pt height 6pt\hfill}\par}

\begin{document}

\title{Proof System for Plan Verification under 0-Approximation Semantics
}
\author{
Xishun Zhao \footnote{Corresponding author. Tel: 0086-20-84114036,
Fax:0086-20-84110298.}
\thanks{This research was partially supported by the NSFC project
under grant number: 60970040 and a MOE project
under grant number: 05JJD72040122. }, Yuping Shen\\
Institute of Logic and Cognition,
\\ Sun Yat-sen
University\\ 510275 Guangzhou, (P.R. China)\\
{Email: hsszxs@mail.sysu.edu.cn}  }

\maketitle

\begin{abstract}
In this paper a proof system is developed for plan verification problems $\{X\}c\{Y\}$ and $\{X\}c\{\mbox{KW} p\}$ under 0-approximation semantics for ${\mathcal A}_K$. Here, for a plan $c$, two sets $X,Y$ of fluent literals, and a literal $p$, $\{X\}c\{Y\}$ (resp. $\{X\}c\{\mbox{KW} p\}$) means that all literals of $Y$ become true (resp. $p$ becomes known) after executing $c$ in any initial state in which all literals in $X$ are true.
Then, soundness and completeness are proved. The proof system allows verifying plans and generating plans as well.
\end{abstract}


\section{Introduction}

{Planning} refers to the procedure of finding {a sequence of actions}(i.e., a \emph{plan}) which
leads a possible world from an initial state to a {goal}. In the early days of Artificial Intelligence(AI),
an agent(i.e., plan generator or executor) was assumed to have complete knowledge about the world but it turned out to
 be unrealistic. Therefore, planning under \emph{incomplete knowledge} earns a lot of attention since late
 1990s \cite{Moore85,Etzioni92anapproach,DBLP:conf/aaai/ScherlL93,Levesque96-WhatPlanning,Petrick04extendingthe,oglietti05incomplete}.
A widely accepted solution is to equip the planner with actions for producing knowledge, also called \emph{sensing actions},
    and allow to use \emph{conditional plan}\cite{Levesque96-WhatPlanning,sb,citeulike:4295998,Scherl:2003:KAF:767243.767244,DBLP:conf/lpnmr/NieuwenborghEV07}, i.e.,
    plans containing \emph{conditional expressions} (e.g., \textbf{If}-\textbf{Then}-\textbf{Else} structures).

Consider the following example \cite{sb}, say a bomb can only be safely defused  if its alarm is switched off. Flipping the switch causes the alarm off if it is on and vice versa. At the beginning we only know  the bomb is not \emph{disarmed} and not \emph{exploded}, however, we do not know whether or not the alarm is on, i.e., the knowledge about initial state of the domain is incomplete. An agent could correctly defuse the bomb by  performing the conditional plan $c$ below:
$$check;\mathbf{If}\ alarm\_off\ \mathbf{Then}\ defuse\ \mathbf{Else}\ \{switch; defuse\}$$
in which $check$ is a sensing action that produces the knowledge about the alarm. It is necessary to mention that there exists no feasible classical  plans for this scenario, e.g., neither $defuse$ nor $switch;defuse$ could safely disarm the bomb.

To describe and reason about  domains with incomplete knowledge,  a number of logical frameworks  were proposed  in the literature.
One of well-established formalizations is the \emph{action language} $\mathcal{A}_K$ \cite{sb,cvt}.
In contrast to its first order antecedents \cite{Moore85,DBLP:conf/aaai/ScherlL93}, $\mathcal{A}_K$ possesses a natural syntax and a transition function based semantics, both together provides a flexible mechanism to model the change of an agent's knowledge in a simplified Kripke structure.

In  \cite{sb} the authors propose several semantics for ${\mathcal A}_K$, all of which, roughly speaking, are based on some transition function from pairs of actions and initial states to states.
For convenience we use SB-semantics to denote the semantics based on the transition function which maps pairs of actions and c-states to c-states. Here, a c-state is a pair of a world state and a knowledge state which is a set of world states. One of the results in \cite{cvt} is that the polynomial plan existence
problem under SB-semantics is PSPACE-complete. Even we restrict the number of fluents determined by a sensing action, the existence of polynomial plan with limited number sensing actions is $\Sigma^P_2$-complete \cite{cvt}. To overcome the high complexity, Baral and Son \cite{sb} have proposed $i$-approximations, $i=0,1,\cdots$. It has been proved in \cite{cvt} that under some restricted conditions polynomial plan existence problem under 0-approximation is NP-complete, that is, it is still intractable because it is widely believed that there is no polynomial algorithm solving an NP-complete problem.

Although modern planers are quite successful to produce and verify short plans they still face a great challenge to generate longer plans. There have been many efforts to construct transformations from planning or plan verification to other logic formalisms, for example,
first-order logic (FOL)  \cite{sc07lin,DBLP:conf/ijcai/Kartha93,sb}, propositional satisfiability (SAT) \cite{satHandbookplanning}, QBF satisfiability (QSAT), \cite{Otwell04aneffective,2005safeplanningqbf}, non-monotonic logics \cite{Gelfond93representingaction,Baral:1993:RCA:1624140.1624145,Lin:1995:PCT:201019.201021}, and so on. These approaches provide ways to use existing solvers for planning and plan verification, they do not, however, tell us how to generate and verify new plans from old ones.

It is well known that programming is generally also very hard, however, proof system for program verification allows one to construct new correct programs from shorter ones \cite{apt}. Similarly, proof systems for plan verification would be helpful for verifying and constructing longer correct plans.

For a given domain description $D$, two sets $X, Y$ of fluent literals, and a plan $c$,  we consider the verification problem of determining whether $D\models\{X\}c\{Y\}$, that is, whether all literals of $Y$ becomes true after executing $c$ in any initial state in which all literals of $X$ are true. It seems natural that from $D\models \{X\}c_1\{Y\}$ and $D\models\{Y\}c_2\{Z\}$ we should obtain $D\models \{X\}c_1;c_2\{Z\}$. That is, $$\frac{\{X\}c_1\{Y\},\ \{Y\}c_2\{Z\}}{\{X\}c_1;c_2\{Z\}}$$ should be a valid rule. This paper is devoted to develop a sound and complete proof system for plan verification under 0-approximation.

One important observation is that constructing proof sequences could also be considered as a procedure for generating plans. This feature is very useful for the agent
to do so-called \emph{off-line} planning \cite{Lin97howto,Cadoli:1997:SKC:1216075.1216081}. That is, when the agent is free from assigned tasks, she
could continuously  compute (short) proofs and store them into a well-maintained database. Such a database consists of a huge number of proofs of the form
$\{X\}c\{Y\}$ after certain amount of time. W.l.o.g., we may assume these proofs are stored into a graph, where  $\{X\}$, $\{Y\}$ are nodes and $c$ is an connecting edge.  With such a database, the agent could do \emph{on-line} query quickly. Precisely speaking, asking whether a plan $c'$ exists for leading state $\{X'\}$ to
$\{Y'\}$, is equivalent to look for a path $c'$ from $\{X'\}$ to $\{Y'\}$ in the graph. This  is known as the PATH problem and could be easily computed (NL-complete, see \cite{ccama}).

The paper is organized as follows. In Section 2 we mainly recall the language of ${\mathcal A}_K$ and the 0-approximation semantics. In addition, a few new lemmas are proved, which will be used in later sections. Section 3 is devoted to the construction of proof system. Soundness and completeness are proved. Section 4 concludes this paper.

\section{The Language ${\cal{A}}_K$}
\label{sec-intro-ak}

The language ${\cal{A}}_K$ \cite{sb} proposed by Baral \& Son is a well known framework for  reasoning about sensing
actions and conditional planning. In this section we recall the syntax and the 0-approximation semantics of \AK, in addition we prove several new properties (e.g. the monotonicity of 0-transition function, see Lemma \ref{lemma2.1} below) which will be used in next section.

\subsection{Syntax of \AK}

Two disjoint non-empty sets of symbols, called {\em fluent names} (or {\em fluents}) and {\em action names}
(or {\em actions}) are introduced as the alphabet of the language $\AK$. A {\em fluent literal} is either a fluent $f$
or its negation $\neg f$. For a fluent $f$, by ${\neg \neg f}$ we mean $f$. For a fluent literal $p$,
we define fln$(p):=f$ if $p$ is a fluent $f$ or is $\neg f$. Given a set $X$ of fluent literals, $\neg X$ is defined as $\{\neg p \mid p\in X\}$, and fln($X$) is defined as $\{\mbox{fln}(p) \mid  p\in X\}$.

The language \AK  uses four kinds \emph{propositions} for describing a domain.

An {\em initial-knowledge proposition} (which is called v-proposition in  \cite{sb}) is an expression of the
form \begin{equation} \mbox{{\bf initially} } p\end{equation}
where $p$ is a fluent literal. Roughly speaking, the above  proposition says that $p$ is
initially known to be true.

An {\em effect proposition} (\emph{ef-proposition} for short) is an expression of the form
\begin{equation} a \mbox{\bf{ casues} }p \mbox{{ \bf if}
}p_1,\cdots,p_n
\end{equation}
where $a$ is an action and $p$, $p_1,\cdots, p_n$ are fluent literals.
We say $p$ and $\{p_1,\cdots,p_n\}$ are the {\em effect} and the {\em precondition} of the proposition, respectively.
The intuitive meaning of the above proposition is that $p$ is guaranteed to
be true after the execution of action $a$ in any state of the
world where $p_1,\cdots,p_n$ are true. If the precondition is
empty then we drop the {\bf if} part and simply say: $a$
{\bf causes} $p$.

An {\em executability proposition} ({\em ex-proposition} for short) is an expression of the form
\begin{equation}\mbox{{\bf executable} }a \mbox{ {\bf if }
}p_1,\cdots, p_n\end{equation}
where $a$ is an action and
$p_1,\cdots, p_n$ are fluent literals. Intuitively, it says that the action $a$ is executable
whenever $p_1, \cdots, p_n$ are true. For convenience, we call $\{p_1,\cdots,p_n\}$ the ex-preconditions of the proposition.

A {\em knowledge proposition} ({\em k-proposition} for short) is of the form
\begin{equation}a \mbox{ {\bf determines} }f\end{equation}
where
$a$ is an action and $f$ is a fluent. Intuitively, the above
proposition says that after $a$ is executed the
agent will know whether $f$ is true or false.

A \emph{proposition} is either an initial-knowledge proposition, or an
ef-proposition, or an ex-proposition, or a k-proposition.  Two initial-knowledge propositions {\bf initially} $f$ and {\bf
initially} $g$ are called {\em contradictory} if $f=\neg g$. Two effect propositions ``$a$ {\bf causes} $f$ {\bf if} $p_1,\cdots,p_n$"
and ``$a$ {\bf causes} $g$ {\bf if} $q_1,\cdots,q_m$"  are
called  {\em contradictory} if $f=\neg g$ and $\{p_1,\cdots,p_n\}\cap\{\neg q_1,\cdots, \neg q_m\}$
is empty.

\begin{definition} \label{def2.1} (\cite{sb}) A domain description in ${\mathcal A}_K$ is a set of propositions
$D$ which does not contain

(1) contradictory initial-knowledge propositions,

(2) contradictory ef-propositions
\end{definition}

Actions occurring in knowledge propositions are called \emph{sensing
actions}, while actions occurring in effect propositions are called
\emph{non-sensing actions}. In this paper we request that for any domain
description $D$ the set of sensing actions in $D$ and the set of
non-sensing actions in $D$ should be disjoint.

%

\begin{definition} \label{def2.2} (Conditional Plan  \cite{sb})
A \emph{conditional plan} is inductively defined as follows:
\begin{enumerate}
\item The empty sequence of actions, denoted by  $[\  ]$, is a
conditional plan;
\item If $a$ is an action then $a$ is a conditional
plan;
\item If $c_1$ and $c_2$ are conditional plans then the
combination $c_1;c_2$ is a conditional plan;
\item If $c_1,\cdots, c_n$ ($n\geq 1$) are conditional plans and
$\varphi_1, \cdots, \varphi_n$ are conjunctions of fluent
literals (which are mutually exclusive but not necessarily
exhaustive) then the following is a conditional plan (also called a {\em case plan}):
$$\mbox{{\bf case}} \ \varphi_1\rightarrow c_1.\ \cdots .\   \varphi_n\rightarrow c_n.\   \mbox{{\bf endcase}} $$
\item Nothing else is a conditional plan.
\end{enumerate}
\end{definition}

Propositions are used to describe a domain, whereas \emph{queries} are used  to ask questions about the domain. For a plan $c$, a set $X$ of fluent literals, and a fluent literal $p$,  we have two kinds of queries:
\begin{equation}\label{fm-k}{\bf Knows}\ X\ {\bf after }\ c\end{equation}
\begin{equation}\label{fm-kw}{\bf Kwhether}\ p\ {\bf after }\ c\end{equation}
Intuitively, query of the form (\ref{fm-k}) asks whether all literals in $X$ will be known to be true after executing $c$, while  query of the form (\ref{fm-kw}) asks
whether $p$ will be either known to be true or known to be false after executing $c$.

\subsection{0-Approximation Semantics}

In this section we arbitrarily fix a domain description $D$ without contradictory propositions.
From now on when we speak of fluent names and action names we mean that they occur in propositions of $D$.

According to  \cite{sb}, an a-state is a pair $(T, F)$ of two disjoint sets of fluent names. A fluent $f$ is true (resp. false) in $(T, F)$ if $f\in T$ (resp. $f\in F$). Dually, $\neg f$ is true (resp. false) if $f$ is false (resp. true). For a fluent name $f$ outside $T\cup F$, both $f$ and $\neg f$ are unknown. A fluent literal $p$ is called possibly true if it is not false (i.e., true or unknown). In the following we often use $\sigma$, $\delta$ to denote a-states. For a set $X=\{p_1,\cdots, p_m\}$ of fluent literals,
we say $X$ is true in an a-state $\sigma$ if and only if every $p_i$ is true in $\sigma$, $i=1,\cdots,m$.

An action $a$ is said to be 0-executable in an a-state $\sigma$ if there exists an
ex-proposition {\bf executable} $a$ {\bf if} $p_1,\cdots,p_n$, such that $p_1,\cdots,p_n$  are true in $\sigma$. The following notations were introduced in \cite{sb}.
\begin{description}
 \item (1) $e^+_a(\sigma):=\{f\mid f$  is a fluent and there exists ``$a$ {\bf causes} $f$ {\bf if } $p_1,\cdots,p_n$''  in $D$ such
that $p_1,\cdots,p_n$ are true in $\sigma\}$.

\item (2) $e^-_a(\sigma):=\{f\mid f$  is a fluent and there exists ``$a$ {\bf causes} $\neg f$ {\bf if } $p_1,\cdots,p_n$''  in $D$ such
that $p_1,\cdots,p_n$ are true in $\sigma\}$.

\item (3) $F^+_a(\sigma):=\{f\mid f$  is a fluent and there exists ``$a$ {\bf causes} $f$ {\bf if } $p_1,\cdots,p_n$''  in $D$ such
that $p_1,\cdots,p_n$ possibly true  in $\sigma\}$.

\item (4) $F^-_a(\sigma):=\{f\mid f$  is a fluent and there exists ``$a$ {\bf causes} $\neg f$ {\bf if } $p_1,\cdots,p_n$''  in $D$ such
that $p_1,\cdots,p_n$ are possible true in $\sigma\}$.

\item (5) $K(a):=\{f \mid f$ is a fluent and ``$a$ {\bf determines} $f$'' is in $D\}$.
\end{description}
For an a-sate $\sigma=(T,F)$ and a non-sensing action $a$ 0-executable in $\sigma$, the result after executing $a$ is defined as
$$\mbox{Res}_0(a,\sigma):=((T\cup e^+_a(\sigma))\setminus F^-_a(\sigma), (F\cup e^-_a(\sigma))\setminus F^+_a(\sigma))$$
The extension order $\preceq$ on a-states is defined as follows  \cite{sb}:
$$(T_1,F_1) \preceq (T_2, F_2)\ \mbox{ if and only if }\ T_1\subseteq T_2, F_1\subseteq F_2.$$
Please note that if $(T_1, F_1)\preceq(T_2,F_2)$ then for a fluent literal $p$ we have
\begin{itemize}
\item if $p$ is true (resp. false) in $(T_1,F_1)$ then $p$ is true (resp. false) in $(T_2, F_2)$,

\item if $p$ is unknown in $(T_2, F_2)$ then $p$ must be unknown in $(T_1,F_1)$, and

\item if $p$ is possibly true in  $(T_2, F_2)$ then $p$ is possibly true in $(T_1,F_1)$.
\end{itemize}
Consequently, for any non-sensing action $a$ and a-states $\sigma_1$ and $\sigma_2$ such that $\sigma_1\preceq \sigma_2$ and $a$ is 0-executable in $\sigma_1$, we have
\begin{itemize}
\item $a$ is 0-executable in $\sigma_2$.
\item $e^+_a(\sigma_1)\subseteq e^+_a(\sigma_2)$, and $e^-_a(\sigma_1)\subseteq e^-_a(\sigma_2)$.
\item $F^+_a(\sigma_2)\subseteq F^+_a(\sigma_1)$, and $F^-_a(\sigma_2)\subseteq F^-_a(\sigma_1)$.
\end{itemize}
Then we have the following proposition.

\begin{proposition} \label{prop2.1}
For any non-sensing action $a$ and a-states $\sigma_1$ and $\sigma_2$ such that $\sigma_1\preceq \sigma_2$ and $a$ is 0-executable in $\sigma_1$, we have
$$\mbox{Res}_0(a,\sigma_1)\preceq \mbox{Res}_0(a,\sigma_2).$$
\end{proposition}

The 0-transition function $\Phi_0$ of $D$ is defined as follows \cite{sb}.
\begin{itemize}
\item If $a$ is not 0-executable in $\sigma$, then $\Phi_0(a,\sigma):=\{\bot\}$.

\item If $a$ is 0-executable in $\sigma$ and $a$ is a non-sensing action, $\Phi_0(a,\sigma):=\{\mbox{Res}_0(a,\sigma)\}$.

\item If $a$ is 0-executable in $\sigma=(T,F)$ and $a$ is a sensing action, then $\Phi_0(a,\sigma):=\{(T',F')\mid (T,F)\preceq (T',F') \mbox{ and }T'\cup F'=T\cup F\cup K(a)\}$.

\item $\Phi_0(a,\Sigma):=\bigcup_{\sigma\in\Sigma} \Phi_0(a,\sigma)$.
\end{itemize}

Let $\Sigma_1, \Sigma_2$ be two sets of a-states, we write $\Sigma_1\preceq\Sigma_2$ if for every a-state $\delta$ in $\Sigma_2$, there is an a-state $\sigma$ in $\Sigma_1$ such that $\sigma\preceq\delta$.

The next proposition follows directly from Proposition \ref{prop2.1}. and the definition of $\Phi_0(a,\sigma)$ above.

\begin{proposition} \label{prop2.2}
Suppose $\sigma_1\preceq\sigma_2$ and $a$ is an action 0-executable in $\sigma_1$, then $\Phi_0(a,\sigma_1)\preceq \Phi_0(a,\sigma_2)$.
\end{proposition}


The extended 0-transition function $\widehat{\Phi}_0$, which
maps pairs of conditional plans and a-states into sets of a-states, is defined inductively as follows.
\begin{definition} \label{def2.3} (\cite{sb})
\begin{description}
\item $\widehat{\Phi}_0([\ ],\sigma):=\{\sigma\}$

\item $\widehat{\Phi}_0(a,\sigma):=\Phi_0(a,\sigma)$

\item When $c$ is a case plan {\bf case} $\varphi_1\rightarrow c_1. \cdots.\ \varphi_k\rightarrow c_k$. {\bf endcase},
 $$\widehat{\Phi}_0(c,\sigma):=\left\{\begin{array}{ll}\widehat{\Phi}_0(c_j, \sigma), &\mbox{if }\varphi_j \mbox{ is true in }\sigma,\\
 \{\bot\},& \mbox{if non of }\varphi_1,\cdots,\varphi_k \mbox{ is true in }\sigma.
 \end{array}\right. $$

\item $\widehat{\Phi}_0(c_1;c_2,\sigma):=
    \bigcup_{\sigma'\in\widehat{\Phi}_0(c_1,\sigma)}\widehat{\Phi}_0(c_2,\sigma')$

\item $\widehat{\Phi}_0(c,\bot):=\{\bot\}$.
\end{description}
$\widehat{\Phi}_0(c,\Sigma):=\bigcup_{\sigma\in\Sigma} \widehat{\Phi}_0(c,\sigma)$.
\end{definition}

\begin{remark} \label{remark2.1}{\em  From the definitions above we know that transition functions $\Phi_0$ and $\widehat{\Phi}_0$ of a domain
description $D$ do not depends on any initial-knowledge proposition.
In other words, if two domain descriptions $D_1$ and $D_2$ contain the same non initial-knowledge propositions,
then their transition functions coincide.}
\end{remark}

A condition plan $c$ is 0-executable in $\sigma$ if $\bot\not\in\widehat{\Phi}_0(c,\sigma)$.

\begin{lemma} \label{lemma2.1} (Monotonicity Lemma) Let $c$ be a plan, $\Sigma_1, \Sigma_2$ be two sets of a-states. Suppose
$\Sigma_1\preceq\Sigma_2$, and $c$ is 0-executable in every a-state on $\Sigma_1$. Then $\widehat{\Phi}_0(c,\Sigma_1)\preceq \widehat{\Phi}_0(c,\Sigma_2)$.
\end{lemma}
\begin{proof}
We proceed by induction on the structure of the plan $c$.
\begin{enumerate}
\item Suppose $c$ consists of only an action $a$. Consider an arbitrary a-state $\sigma'_2 \in \Phi_0(a,\Sigma_2)$. Then there is an a-state $\sigma_2=(T_2,F_2)\in \Sigma_2$ such that $\sigma'_2\in \Phi_0(a,\sigma_2)$. Since $\Sigma_1\preceq \Sigma_2$, pick $\sigma_1=(T_1, F_1)\in \Sigma_1$ such that $\sigma_1\preceq \sigma_2$. It is sufficient to show that $\sigma'_1\preceq \sigma'_2$ for some $\sigma'_1\in\Phi_0(a,\sigma_1)$.

If $a$ is a non-sensing action $a$, then the assertion follows directly from Proposition \ref{prop2.2}.
Suppose $a$ is a sensing action. Then $\sigma'_2$ must be of the form $(T_2\cup X, F_2\cup Y)$ because $a$ is a sensing action, here $X\cup Y=K(a)$. Then clearly $(T_1\cup X, F_1\cup Y)$ must be in $\Phi_0(a,\sigma_1)$. The assertion follows since $(T_1\cup X, F_1\cup Y)\preceq (T_2\cup X, F_2\cup Y)$.

\item Suppose $c$ is case plan  {\bf case} $\varphi_1\rightarrow c_1. \cdots.\  \varphi_k\rightarrow c_k$. {\bf endcase}.
Consider any a-state $\sigma'_2\in \widehat{\Phi}_0(c,\Sigma_2)$. Let $\sigma_1\in \Sigma_1, \sigma_2\in\Sigma_2$ be such that $\sigma_1\preceq\sigma_2$ and $\sigma'_2\in \widehat{\Phi}_0(c,\sigma_2)$. Since $c$ is 0-executable in $\sigma_1$, some  $\varphi_i$ is true in $\sigma_1$. Then $\varphi_i$ is also true in $\sigma_2$ since $\sigma_1\preceq \sigma_2$. Then by the induction hypothesis, $\widehat{\Phi}_0(c,\sigma_1)=\widehat{\Phi}_0(c_i,\sigma_1)\preceq \widehat{\Phi}_0(c_i,\sigma_2)=\widehat{\Phi}_0(c,\sigma_2)$.
Thus, there is $\sigma'_1\in \widehat{\Phi}_0(c,\Sigma_1)$ such that $\sigma'_1\preceq \sigma'_2$. Consequently, $\Phi_0(c,\Sigma_1)\preceq \Phi(c,\Sigma_2)$

\item
Suppose $c=c_1;c_2$. By induction hypothesis $\widehat{\Phi}_0(c_1,\Sigma_1)\preceq \widehat{\Phi}(c_1,\Sigma_2)$. Then by the definition of $\widehat{\Phi}_0$ we have $$\widehat{\Phi}_0(c,\Sigma_1)=\left(\bigcup_{\sigma'\in\widehat{\Phi}_0(c_1,\Sigma_1)}
\widehat{\Phi}_0(c_2,\sigma')\right)
\preceq \left(\bigcup_{\sigma''\in\widehat{\Phi}_0(c_1,\Sigma_2)}
\widehat{\Phi}_0(c_2,\sigma'')\right)=\widehat{\Phi}_0(c,\Sigma_2)$$
\end{enumerate}%
\end{proof}


An a-state $\sigma$ is called an initial a-state of $D$ if $p$ is true in $\sigma$ for any fluent literal $p$ such that the initial-knowledge proposition ``{\bf initially} $p$" is in $D$.

Suppose $D$ is a domain description, $c$ is a conditional plan, $X$ is a set of fluent literals, and $p$ a literals. The semantics for the queries are given below:
\begin{definition} \label{def2.4} (\cite{sb})
\begin{itemize}
\item $D\models_0 {\bf Knows}\ X\ {\bf after}\ c$ if for every initial a-state $\sigma$, the plan $c$ is 0-executable in $\sigma$, and $X$ is true in every a-state in $\widehat{\Phi}_0(c,\sigma)$.
\item $D\models_0 {\bf Kwhether}\ p\ {\bf after}\ c$ if for every initial a-state $\sigma$, the plan $c$ is 0-executable in $\sigma$, and $p$ is either true or false in every a-state in $\widehat{\Phi}_0(c,\sigma)$.
\end{itemize}
\end{definition}

Let $T_D:=\{f\mid ``{\bf initially }\ f"\in D\}$, $F_D:=\{f\mid ``{\bf initially }\ \neg f"\in D\}$. Obviously, $(T_D, F_D)$ is the least initial a-state of $D$, that is, $(T_D,F_D)\preceq \sigma$ for any initial a-state $\sigma$. The following lemma follows easily from Lemma \ref{lemma2.1}.

\begin{lemma} \label{lemma2.2}\
\begin{itemize}
\item $D\models_0 {\bf Knows}\ X\ {\bf after}\ c$ if and only if the plan $c$ is 0-executable in $(T_D, F_D)$, and $X$ true true in every a-state in $\widehat{\Phi}(c,(T_D, F_D))$.
\item $D\models_0 {\bf Kwhether}\ p\ {\bf after}\ c$ if the plan $c$ is 0-executable in $(T_D, F_D)$, and $p$ is either true or false in every a-state in $\widehat{\Phi}(c,(T_D, F_D))$.
\end{itemize}
\end{lemma}

\section{A Proof System for 0-Approximation}

A consistent set $X$ of literals determines a unique a-state $(T_X,F_X)$ by $T_X:=\{f\mid f\in X\}$ and $F_X:=\{f\mid \neg f\in X\}$. And conversely an a-state determines uniquely the set $S_{(T,F)}:=T\cup\neg F$. Obviously, $p\in X$ if and only if $p$ is true in $(T_X,F_X)$ for any literal $p$.

In the following we will not distinguish sets of literals and a-states from each other. For example, Res$_0(a,X)$) is nothing but Res$_0(a,(T_X,F_X))$ which can be regarded as a set of literals. Analogically, we have notations $\Phi_0(c,X)$ and $\widehat{\Phi}_0(c,X)$, which can be regarded as collections of sets of literals.

\begin{definition} \label{def3.1} Let $D$ be a domain description without initial-knowledge propositions. Suppose $X, Y$ are two sets of fluent literals. By
$D\models_0\{X\}c\{Y\}$ we mean $D\cup\mbox{ini}(X)\models_0 {\bf Knows\ } Y {\bf\ after\ }c.$ Here ini$(X)=\{{\bf initially\ }p\mid p\in X\}$.
\end{definition}

\begin{remark} \label{remark3.1} {\em \ \begin{itemize}
\item  The idea of the notation $\{X\} c \{Y\}$ comes from programming verification where in the sense of total correctness $\{\varphi\} P\{\psi\}$ means that any computation of $P$ starts in a state satisfying $\varphi$ will terminates in a state satisfying $\psi$. (see e.g.  \cite{apt})
\item
By Lemma \ref{lemma2.2}, $D\models_0\{X\}c\{Y\}$ if and only if $Y$ is true in every a-state in $\widehat{\Phi}_0(c,X)$.
\end{itemize}
}\end{remark}

Suppose $D$ is a general domain description (that is, initially-knowledge propositions are allowed). Let $D'$ be the set of all non-initial-knowledge propositions of $D$, and let $X:=\{p\mid ``{\bf initially\ } p" \mbox{ is in } D\}$. Then  $D'\models_0 \{X\} c \{Y\}$ is equivalent to
$D\models_0 {\bf Knows\ } Y\ {\bf after\ } c$.

\subsection{The Proof System PR$^0_D$ for {\bf Knows}} \label{subsection3.1}

In the remainder of this section we fixed a domain description $D$ without initial-knowledge propositions.
We always use $X, Y, X', Y'$ to denote consistent set of fluent literals.
The proof system PR$^0_D$ consists of the following groups of axioms and rules 1-6.
\begin{description}
\item AXIOM 1. (Empty) $$\{X\}[\ ]\{X\}.$$

\item AXIOM 2. (Non-sensing Action)  $$\{X\}a \{(\mbox{Res}_0(a,X))\}.$$
\noindent Where $a$ is a non-sensing action 0-executable in ${X}$.

\item RULE 3. (Sensing Action)  $$\frac{\{X\cup X_1\} c \{Y\},\cdots, \{X\cup X_m\} c \{Y\}}{\{X\} a;c\{Y\}}.$$

\noindent Where $a$ is a sensing action  0-executable in $X$, and $X_1, \cdots X_m$ are all sets $X'$ of fluent literals such that fln$(X')=K(a)$ and $X\cup X'$ is consistent.

\item RULE 4. (Case) $$\frac{\varphi_i\subseteq X,\ \  \{X\}c_i;c'\{Y\}}{\{X\}c;c'\{Y\}}.$$
Where $c$ is the case plan {\bf case} $\varphi_1\rightarrow c_1.\  \cdots\!.\ \varphi_m\rightarrow c_m$.\ {\bf endcase}.

\item RULE 5. (Composition) $$\frac{\{X\}c_1 \{Y'\}, \{Y'\}c_2\{Y\}}{\{X\}c_1;c_2\{Y\}}.$$

\item RULE 6 (Consequence) $$\frac{X'\subseteq X,  \{X'\}c\{Y'\},  Y\subseteq Y'}{\{X\}c\{Y\}}.$$
\end{description}

\begin{definition} \label{def3.2}
A \emph{proof sequence} (or, \emph{derivation}) of PR$^0_D$ is a sequence $\{X_1\}c_1\{Y_1\}, \cdots, \{X_n\}c_n\{Y_n\}$ such that each $\{X_i\}c_i\{Y_i\}$ is either an axiom in PR$^0_D$ or is obtained from some of $\{X_1\}c_1\{Y_1\}, \cdots, \{X_{i-1}\}c_{i-1}\{Y_{i-1}\}$ by applying a rule in PR$^0_D$.

By $D\vdash_0 \{X\} c\{Y\}$, we mean that $\{X\}c\{Y\}$ appears in some proof sequence of PR$^0_D$, that is, $\{X\}c\{Y\}$ can be derived from axioms and rules in PR$^0_D$.
\end{definition}

\begin{example} \label{example3.1} {\em (\cite{sb})
Let
$$D:=\left\{\begin{array}{l}
check\ \mathbf{determines}\ alarm\_off\\
defuse\ \mathbf{causes}\ disarmed\ \mathbf{if}\ alarm\_off\\
defuse\ \mathbf{causes}\ exploded\ \mathbf{if}\ \neg alarm\_off\\
switch\ \mathbf{causes}\ \neg alarm\_off\ \mathbf{if}\ alarm\_off \\
switch\ \mathbf{causes}\  alarm\_off\ \mathbf{if}\ \neg alarm\_off \\
{\bf executable}\ check\ {\bf if}\ \neg exploded\\
{\bf executable}\ switch\ {\bf if}\ \neg exploded\\
{\bf executable}\ defuse\ {\bf if}\ \neg exploded\\
\end{array}\right\}$$

\noindent
Let $c'$ be the case plan: {\bf case}\
\mbox{$\neg alarm\_off\rightarrow switch.\ \
alarm\_off \rightarrow [\ ].\
{\bf endcase},$} and $c$ be the plan: $check; c';
defuse$.
Then the following is a proof sequence of PR$^0_{D}$.
\begin{description}
\item (1) $\{\neg disarmed, \neg exploded, \neg alarm\_off\} switch \{\neg disarmed, \neg exploded, alarm\_off\}$\\ (AXIOM 2)

\item (2) $\{\neg disarmed, \neg exploded, \neg alarm\_off\} c' \{\neg disarmed, \neg exploded, alarm\_off\}$\\ ((1) and RULE 4)

\item (3) $\{\neg disarmed, \neg exploded, alarm\_off\} [\ ] \{\neg disarmed, \neg exploded, alarm\_off\}$\\ (AXIOM 1)

\item (4) $\{\neg disarmed, \neg exploded, alarm\_off\} c' \{\neg disarmed, \neg exploded, alarm\_off\}$\\ ((3) and RULE 4)

\item (5) $\{\neg disarmed, \neg exploded\} check; c' \{\neg disarmed, \neg exploded, alarm\_off\}$\\ ((2), (4) and RULE 3)

\item (6) $\{\neg disarmed, \neg exploded, alarm\_off\} defuse \{disarmed, \neg exploded, alarm\_off\}$\\ (AXIOM 2)

\item (7) $\{\neg disarmed, \neg exploded\} c \{disarmed, \neg exploded, alarm\_off\}$\\ ((6) and RULE 5)
\end{description}
}
\end{example}
\begin{remark} \label{remark3.2} {\em
One important observation is that constructing a proof sequence could also  be considered as a procedure for generating plans. This feature is very useful for the agent
to do so-called \emph{off-line} planning \cite{Lin97howto,Cadoli:1997:SKC:1216075.1216081}. That is, when the agent is free from assigned tasks, she
could continuously  compute (short) proofs and store them into a well-maintained database. Such a database consists of a huge number of proofs of the form
$\{X\}c\{Y\}$ after certain amount of time. W.l.o.g., we may assume these proofs are stored into a graph, where  $\{X\}$, $\{Y\}$ are nodes and $c$ is an connecting edge.  With such a
database, the agent could do \emph{on-line} query quickly. Precisely speaking, asking whether a plan $c'$ exists for leading state $\{X'\}$ to
$\{Y'\}$, is equivalent to look for a path $c'$ from $\{X'\}$ to $\{Y'\}$ in the graph. This  is known as the PATH problem and could be easily computed (NL-complete, see \cite{ccama}).
}
\end{remark}

\subsubsection{Soundness of PR$^0_D$}

\begin{theorem} {\em (Soundness of PR$^0_D$)} \label{theorem3.1}
PR$^0_D$ is sound. That is, for any conditional plan $c$ and any consistent sets $X, Y$ of fluent literals, $D\vdash_0 \{X\}c\{Y\}$ implies $D\models_0\{X\}c\{Y\}$.
\end{theorem}
\begin{proof}
Suppose $D\vdash_0 \{X\}c\{Y\}$. Then $\{X\}c\{Y\}$ has a derivation. We shall proceed by induction on the length of the derivation.
Let $\Phi_0$ and $\widehat{\Phi}_0$ be 0-transition functions of $D$. Please note that for any set $S$ of fluent literals, the 0-transition functions of $D\cup\mbox{ini}(S)$ are the same as $\Phi_0$ and $\widehat{\Phi}_0$, respectively (see Remark \ref{remark2.1}).

\begin{enumerate}

\item Suppose $\{X\}c\{Y\}$ is an axiom in AXIOM 1. Then $X=Y$ and $c=[\ ]$. Clearly, $D\models_0\{X\}[\ ]\{X\}$.

\item Suppose $\{X\}c\{Y\}$ is an axiom in AXIOM 2, i.e., $c$ consists of only a non-sensing action $a$ which is 0-executable in $X$, and  $Y=\mbox{Res}_0(a,X)$. Since $\widehat{\Phi}_0(a,X)=\{\mbox{Res}_0(a,X)\}$, it follows that $D\models_0\{X\}a\{Y\}$.

\item Suppose $\{X\}c\{Y\}$ is obtained by applying a rule in RULE 3. Then $c=a;c_1$ for some sensing action $a$ 0-executable in $X$, and $\{X\}c\{Y\}$ is obtained from $\{X\cup X_1\}c_1\{Y\}$, $\cdots$, $\{X\cup X_m\}c_1\{Y\}$, where $X_1, \cdots X_m$ are all sets $X'$ of fluent literals such that fln$(X')=K(a)$ and $X\cup X'$ is consistent. By the induction hypothesis, $$D\models_0\{X\cup X_i\}c_1\{Y\}, \mbox{\ \  for }i=1,\cdots,m.$$
That is, all literals in $Y$ are true in every set in $\widehat{\Phi}_0(c_1, X\cup X_i)$.
Please note that $\Phi_0(a,X)=\{X\cup X_1, \cdots, X\cup X_m\}$.  By the definition of $\widehat{\Phi}_0$ (see Definition \ref{def2.3}),
$$\widehat{\Phi}_0(c,X)=\bigcup_{i=1}^m\widehat{\Phi}_0(c_1, X\cup X'_i).$$
Therefore, $D\models_0 \{X\}c\{Y\}$.

\item Suppose $\{X\}c\{Y\}$ is obtained by applying a rule in RULE 4. That is, $c$ is a plan $c_1;c_2$, where $c_1$ is a case plan ${\bf case\ } \varphi_1\rightarrow c'_1.\ \cdots.\ \varphi_n\rightarrow c'_n. \ {\bf endcase}$ such that for some $i\in\{1,\cdots,n\}$, $\varphi_i\subseteq X$ and $\{X\}c'_i;c_2\{Y\}$ has been derived.
By the induction hypothesis, we have $D\models_0\{X\}c'_i;c_2\{Y\}$.
By Definition \ref{def2.3}, we have $\widehat{\Phi}_0(c,X)=\widehat{\Phi}_0(c_2,\widehat{\Phi}_0(c_1,\sigma))=
\widehat{\Phi}_0(c_2,\widehat{\Phi}_0(c'_i,X))=\widehat{\Phi}_0(c'_i;c_2, X)$. Then, all literals of $Y$ are true in $\widehat{\Phi}_0(c,X)$.
Thus, $D\models_0 \{X\}c\{Y\}$.

\item Suppose $\{X\}c\{Y\}$ is obtained from $\{X\}c_1\{Y'\}$ and $\{Y'\}c_2\{Y\}$ by applying a rule in RULE 5. By the inductive hypothesis,
$$D\models_0 \{X\}c_1\{Y'\} \mbox{   and   }D\models_0\{Y'\}c_2\{Y\}.$$

Then for any $S\in\widehat{\Phi}_0(c_1,X)$, we have $Y'\subseteq S$ (i.e., $(T_{Y'}, F_{Y'})\preceq (T_S,F_S)$). Thus, by Lemma \ref{lemma2.1},  $\widehat{\Phi}_0(c_2, Y')\preceq \widehat{\Phi}_0(c_2,S)$. Then
$$\widehat{\Phi}_0(c_2,Y')\preceq \left(\bigcup_{S\in\widehat{\Phi}_0(c_1,X)}\widehat{\Phi}_0(c_2,S)\right)=\widehat{\Phi}_0(c,X),$$
It follows that
$D\models_0 \{X\}c\{Y\}$.

\item Suppose $\{X\}c\{Y\}$ is obtained by applying a rule in RULE 6. That is, there is $X'\subseteq X$ and $Y'\supseteq Y$ such that $\{X'\}c\{Y'\}$ has been derived. Then by the induction hypothesis, all literals in $Y'$ is known to be true in $\widehat{\Phi}_0(c,X')$, so are literals in $Y$. By Lemma \ref{lemma2.1} we have $\widehat{\Phi}_0(c,X')\preceq \widehat{\Phi}_0(c,X)$. Therefore, $D\models_0 \{X\}c\{Y\}$.
\end{enumerate}
Altogether, we complete the proof.
\end{proof}

\subsubsection{Completeness of PR$^0_D$}

\begin{theorem} \label{theorem3.2} (Completeness of PR$^0_D$)
PR$^0_D$ is complete. That is,  for any conditional plan $c$ and any consistent sets $X, Y$ of fluent literals, $D\models_0\{X\}c\{Y\}$ implies $D\vdash_0 \{X\}c\{Y\}$.
\end{theorem}
\begin{proof}  Suppose $D\models_0\{X\}c\{Y\}$. We shall show $D\vdash_0 \{X\}c\{Y\}$. We shall proceed by induction on the structure of $c$.
\begin{enumerate}
\item Suppose $c$ consists of only an action $a$. Then $a$ is 0-executable in $X$.
\begin{itemize}
\item {\bf Case 1.} $a$ is a non-sensing action. Then all literals in $Y$ are true in Res$_0(a,X)$, that is, $Y\subseteq$ Res$_0(a,X)$.  By Axiom 2, $D\vdash_0\{X\}a\{\mbox{Res$_0(a,X)$}\}$. Then by RULE 6, we obtain $D\vdash_0\{X\}a\{Y\}$.

\item {\bf Case 2.} $a$ is a sensing action. Consider any $p\in Y$. We shall show $p\in X$. Suppose otherwise, then $X':=X\cup\{\neg p\}$ is still consistent. Then $\Phi_0(a,X)\preceq\Phi_0(a,X')$. Thus $p$ should also be true in every a-state in $\Phi_0(a,X')$. On the other hand, $\neg p$ is true in every a-state in $\Phi_0(a,X')$ since $\neg p\in X'$.  This is a contradiction. Thus $Y\subseteq X$. Then for any set $X'$ such that fln$(X')=K(a)$ and $X\cup X'$ is consistent, we have $D\vdash_0\{X\cup X'\}[\ ]\{Y\}$. Now applying RULE 3 we obtain   $D\vdash_0\{X\}a\{Y\}$.
\end{itemize}

\item Suppose $c$ is a case plan {\bf case} $\varphi_1\rightarrow c_1.\  \cdots .\  \varphi_m\rightarrow c_m$.\ {\bf endcase}.
Since $D\models_0\{X\}c\{Y\}$, it follows that  $\varphi_i\subseteq X$ for some $i$ (otherwise, $c$ would not be 0-executable in $X$). Then $D\models_0 \{X\}c_i\{Y\}$. By the induction hypothesis, $D\vdash_0\{X\}c_i\{Y\}$. By RULE 4 we have $D\vdash_0\{X\}c\{Y\}$.

\item Suppose $c$ is a composition plan $c_1;c_2$. We shall show the assertion by induction on the structure of $c_1$.
\begin{itemize}
\item $c_1$ is a non-sensing action $a$.
By Definition \ref{def2.3}, $\widehat{\Phi}_0(a;c_2, X)=\widehat{\Phi}_0(c_2,\mbox{Res}_0(a,X))$.
By the induction hypothesis,  $D\vdash_0\{\mbox{Res}_0(a,X)\}c_2\{Y\}$. By AXIOM 2 and RULE 5, we obtain $D\vdash_0\{X\}c\{Y\}$.

\item
$c_1$ is a sensing action $a$. Consider any $X'$ such that fln$(X')=K(a)$ and $X\cup X'$ is consistent. Since $D\models_0 \{X\}a;c_2\{Y\}$, it follows $D\models_0 \{X\cup X'\}c_2\{Y\}$.
Then by the induction hypothesis we have $D\vdash_0\{X\cup X'\}c_2\{Y\}$.
By RULE 3 we obtain $D\vdash_0 \{X\}a;c_2\{Y\}$.

\item $c$ is a case plan {\bf case} $\varphi_1\rightarrow c'_1.\  \cdots.\  \varphi_m\rightarrow c'_m$.\ {\bf endcase}.
Since $c$ is 0-executable in $X$, it follows that $\varphi_i\subseteq X$ for some $i$. Then $D\models_0 \{X\}c'_i;c_2\{Y\}$. By the induction hypothesis. $D\vdash_0 \{X\}c'_i;c_2\{Y\}$. By RULE 4 we have $D\vdash_0\{X\}c_1;c_2\{Y\}$.

\item $c_1$ is $c'_1;c''_1$ such that $c'$ and $c''$ are not empty. Then $c$ is $c'_1;(c''_1;c_2)$. Now $c'_1$ is shorter. By the induction hypothesis, $D\vdash_0\{X\}c\{Y\}$.
\end{itemize}
\end{enumerate}
Altogether, we complete the proof.
\end{proof}

\subsection {The Proof System PRKW$^0_D$ for {\bf Knows-Whether}}

In this section we shall construct a proof system for reasoning about {\bf Kwhether} $p$ {\bf after} $c$ (here $p$ is a fluent literal). We also fix an arbitrary domain description $D$ without initial knowledge-propositions. Similar to the notation $\{X\}c\{Y\}$, we introduce notation $\{X\}c\{\mbox{KW}p\}$.

\begin{definition} \label{def3.3} Let $c$ be a plan, $X$ be a consistent set of fluent literals, and $p$ a fluent literal.
By $D\models_0 \{X\}c\{\mbox{KW}p\}$ we mean $$D\cup\mbox{ini}(X)\models_0 {\bf Kwhether}\ p\ {\bf after}\ c.$$
\end{definition}

Proof system
PRKW$^0_D$ consists of axioms and rules of groups 1-6 in Section \ref{subsection3.1} and the following groups 7-12.

\begin{description}
\item AXIOM 7. $$\{X\}a\{\mbox{KW}f\}$$

Where
$a$ is a sensing action 0-executable in $X$, and $f$ is a fluent name such that the k-proposition ``$a$ {\bf determines} $f$" belongs to $D$.

\item RULE 8. $$\frac{\{X\}c\{\{p\}\}}{\{X\}c\{\mbox{KW}p\}}$$

\item RULE 9. $$\frac{\{X\}c\{\mbox{KW}p\}}{\{X\}c\{\mbox{KW}\neg p\}}$$

\item RULE 10. (Sensing Action)  $$\frac{\{X\cup X_1\} c \{\mbox{KW}p\},\cdots, \{X\cup X_m\} c \{\mbox{KW}p\}}{\{X\} a;c\{\mbox{KW}p\}}.$$

\noindent Where $a$ is a sensing action 0-executable in $X$, and $X_1, \cdots X_m$ are all sets $X'$ of fluent literals such that fln$(X')=K(a)$ and $X\cup X'$ is consistent.


\item RULE 11. (Composition) $$\frac{\{X\}c_1\{Y\},\  \{Y\}c_2\{\mbox{KW}p\}}{\{X\}c_1;c_2\{\mbox{KW}p\}}$$

\item RULE 12. (Case) $$\frac{\varphi_i\subseteq X,\ \{X\}c_{i};c' \{\mbox{KW}p\}}{\{X\}c;c'\{\mbox{KW}p\}}.$$
Where $c$ is the case plan {\bf case} $\varphi_1\rightarrow c_1. \cdots.\  \varphi_n\rightarrow c_n$. {\bf endcase}.
\end{description}

\begin{definition}[Proof Sequence of PRKW$^0_D$] \label{def3.4}
A Proof sequence (or, derivation) of PRKW$^0_D$ is a sequence of elements with the form $\{S_1\}c_1\{T \}$ or $\{S\}c\{\mbox{KW}p\}$ such that each element is either an axiom in PRKW$^0_D$ or is obtained from some of previous elements by applying a rule in PRWK$^0_D$.

By $D\vdash_0 \{S\} c\{\mbox{KW}p\}$, we mean that $\{S\}c\{\mbox{KW}p\}$ appears in some proof sequence of PRKW$^0_D$, that is $\{S\}c\{\mbox{KW}p\}$ can be derived from axioms and rules in PRKW$^0_D$.
\end{definition}

\begin{remark} \label{remark3.3} {\em Please note that $\{X\}c\{\mbox{KW}p\}$ never appears as a premise in a rule with consequence of the form $\{X'\}c'\{Y'\}$. Thus, $\{X\}c\{Y\}$ is derivable in PRKW$^0_D$ if and only if it is derivable in PR$^0_D$. So, for derivability of $\{X\}c\{Y\}$ in PRKW$^0_D$, we still employ the notation $D\vdash_0\{X\}c\{Y\}$.}
\end{remark}


\begin{theorem} \label{theorem3.3} (soundness of PRKW$^0_D$) Given a plan $c$, then $D\vdash_0\{X\}c\{\mbox{KW}p\}$ implies $D\models_0\{X\}c\{\mbox{KW}p\}$ for any consistent set $X$ of fluent literals, and any fluent literal $p$.
\end{theorem}
\begin{proof}
We can show this theorem by induction on the length of derivations. By the soundness of PR$^0_D$,
there are six cases according to whether $\{S\}c\{\mbox{KW}p\}$ is an axiom in AXIOM 7 or obtained by applying a rule in group 8-12.
For each case, the proof is easy. We omit the proof.
\end{proof}


\begin{theorem} \label{theorem3.4} (completeness of PRKW$^0_D$) Given a plan $c$, then $D\models_0\{X\}c\{\mbox{KW}p\}$ implies $D\vdash_0\{X\}c\{\mbox{KW}p\}$ for any consistent set $X$ of fluent literals, and any fluent literal $p$.
\end{theorem}
\begin{proof}We proceed by induction on the structure of $c$.
Suppose $D\models_0\{X\}c\{\mbox{KW}p\}$.
\begin{enumerate}
\item $c$ is empty. Then it must be that $p\in X$ or $\neg p\in X$. Then $\{X\}[\ ]\{\{p\}\}$ or $\{X\}[\ ]\{\{\neg p\}\}$ is derivable. Then by RULE 8-9 we can derive $\{X\}[\ ]\{\mbox{KW}p\}$.

\item $c$ consists of only a sensing action $a$. Then $a$ is 0-executable in $X$. If $p\in X$, it is clearly that $\{X\}a\{\{p\}\}$ is derivable. From RULE 8 we derive $\{X\}a\{\mbox{KW}p\}$.
By the same argument, if $\neg p\in X$, then $D\vdash_0\{X\}a\{\mbox{KW}\neg p\}$, and then we can derive $\{X\}a\{\mbox{KW}p\}$ by applying RULE 9.
Now we suppose neither $p$ nor $\neg p$ is in $X$. We claim that the k-proposition ``$a$ {\bf determines} fln($p$)" belongs to $D$ (Otherwise, $p$ and $\neg p$ would remain
unknown in every a-state in $\Phi_0(a,X)$. This contradicts the assumption $D\models_0\{X\}a\{\mbox{KW}p\}$). Now we have an axiom $\{X\}a\{\mbox{KW }\mbox{fln}(p)\}$. If $p$ itself is a fluent name then we are down, else we derive $\{X\}c\{\mbox{KW}p\}$ by applying RULE 9.

\item $c$ consists of only a non-sensing action $a$. Since $D\models_0\{X\}a\{\mbox{KW}p\}$,  it follows that
$a$ is 0-executable $X$ and either $p$ or $\neg p$ is true in Res$_0(a,X)$.
That is, $p\in \mbox{Res}(a,X)$ or $\neg p\in\mbox{Res}(a,X)$. Since  $D\vdash_0\{X\}a\{\mbox{Res}(a,X)\}$, we have $D\vdash_0\{X\}a\{\{p\}\}$ or $D\vdash_0\{X\}a\{\{\neg p\}\}$.
Then either $\{X\}a\{\mbox{KW}p\}$ or $\{X\}a\{\mbox{KW}\neg p\}$ can be derived by applying RULE 8. If $\{X\}a\{\mbox{KW}\neg p\}$ is derivable then we obtain $\{X\}a\{\mbox{KW} p\}$ by applying RULE 9.

\item $c$ is a case plan of the form {\bf case} $\varphi_1\rightarrow c_1.\ \cdots.\ \varphi\rightarrow c_n$. {\bf endcase}.
Then there must be some $i\in\{1,\cdots,n\}$ such that $\varphi_i\subseteq X$. Otherwise, $c$ would not be 0-executable. Then we can see that $D\models_0\{X\}c_i\{\mbox{KW}p\}$. By the induction hypothesis, we have $D\vdash_0\{X\}c_i\{\mbox{KW}p\}$. Then we can derive $\{X\}c\{\mbox{KW}p\}$ by RULE 12.

\item Suppose $c=c_1;c_2$ such that $c_1$and $c_2$ are non-empty. We show $D\vdash_0\{X\}c\{\mbox{KW}p\}$ by induction on the structure of $c_1$.

\begin{itemize}
\item $c_1$ is a sensing action $a$. Let $X_1, \cdots X_m$ be all sets $X'$ of fluent literals such that fln$(X')=K(a)$ and $X\cup X'$ is consistent. Consider an arbitrary $X_i$. We have $D\models_0\{X\cup X_i\}c_2\{\mbox{KW}p\}$ since $\widehat{\Phi}_0(c_2,X\cup X_i)\subseteq \widehat{\Phi}_0(a;c_2,X)$. By the induction hypothesis, $D\vdash_0\{X\cup X_i\}c_2\{\mbox{KW}p\}$. Now by RULE 10 we can derive $\{X\}a;c_2\{\mbox{KW}p\}$.

\item $c_1$ is a non-sensing action $a$. Then $a$ is 0-executable in $X$.
Since $\widehat{\Phi}_0(c_2,\mbox{Res}_0(a,X))=\widehat{\Phi}_0(a;c_2,X)$, it follows that
$D\models_0\{\mbox{Res}_0(a,X)\}c_2\{\mbox{KW}p\}$.
By the induction hypothesis, $\{\mbox{Res}(a,X)\}c_2\{\mbox{KW}p\}$ is derivable. Please note that $\{X\}a\{\mbox{Res}(a,X)\}$ is an axiom in AXIOM 2. By RULE 11, we can derive $\{X\}a;c_2\{\mbox{KW}p\}$.

\item $c_1$ is a case plan {\bf case} $\varphi_1\rightarrow c'_1.\ \cdots.\ \varphi_n\rightarrow c'_n$. {\bf endcase}. We know that $\varphi_i\subseteq X$ for some $i\in\{1,\cdots,n\}$. It follows that $D\models_0\{X\}c'_i;c_2\{\mbox{KW}p\}$ since we have assumed $D\models_0\{X\}c_1;c_2\{\mbox{KW}p\}$.
By the induction hypothesis, $D\vdash_0\{S\}c'_i;c_2\{\mbox{KW}p\}$. Now applying RULE 12 we can derive $\{X\}c_1;c_2\{\mbox{KW}p\}$.

\item $c_1=c'_1;c'_2$ such that $c'_1, c'_2$ are not empty plan. Then $c=c'_1;(c'_2; c_2)$. Now $c'_1$ is shorter. Then $\{X\}c\{\mbox{KW}p\}$ is derivable by the induction hypothesis.
\end{itemize}
\end{enumerate}
Altogether, we complete the proof.
\end{proof}

\section{Conclusions}

In this paper, we have proposed a proof system for plan verification under 0-approximation semantics introduced in \cite{sb}.
The proof system  has the following advantages:
it is self-contained, hence it does not rely on any particular logic, and need not to pay extra costs to the process of translation; it could be used for both plan verification or plan generation. Particularly, we would like to point out that proof system based inference approach possesses a very desirable property for \emph{off-line} planning. Simply speaking, it allows the agent
to produce and store (shorter) proofs  into a database in spare time,  and perform quick \emph{on-line} planning by constructing requested proofs from
the (shorter) proofs in the database.

Please note that the construction of the proof systems PRKW$^0_D$ depends essentially on the monotonicity property of $\Phi_0$ (see Lemma \ref{lemma2.1}).
According to \cite{sb}, an action $a$ is 1-executable in an a-state $\sigma$ if it is 0-executable in every complete a-state extending $\sigma$. And if a non-sensing action $a$ is 1-executable in $\sigma$, then Res$_1(a,\sigma)$ is defined as the intersection of all $\mbox{Res}_0(a,\sigma')$, $\sigma'\in \mbox{Comp}(\sigma)$ which is the set of all complete a-states extending $\sigma$.
Obviously, Res$_1$ is monotonic, that is, if $\sigma\preceq\delta$ then $\mbox{Res}_1(a,\sigma)\preceq \mbox{Res}_1(a,\delta)$. Thus the transition function $\Phi_1$ and $\widehat{\Phi}_1$ (for precise definition please see \cite{sb}) are also monotonic.
Therefore, in PRKW$^0_D$, if we replace $\{X\}a\{\mbox{Res}_0(a,X)\}$ in AXIOM 2 by $\{X\}a\{\mbox{Res}_1(a,X)\}$, and replace in all groups ``0-executable'' by ``1-executable'', we will obtain a sound and complete proof system PRKW$^1_D$ for plan verification under 1-approximation. Please note, however, since 1-exeutability is unlikely solvable in poly-time, to determine whether a rule in PRKW$^1_D$ is applicable seems intractable.

The work of Matteo Baldoni \emph{et al}  \cite{Baldoni:2001:RCA:646293.687243} is closely related to our idea. They proposed
a modal logic approach  for reasoning about sensing actions, together with goal directed proof procedure for generating conditional plans. The states of a world are represented in  \cite{Baldoni:2001:RCA:646293.687243} as three valued models, so queries about {Knows-Whether} are not
supported. Moreover, their approach does not provide reasoning about case plan, and the completeness of their proof
procedure is unknown.

In the future, we shall further work on proof system for more powerful action logics. We shall consider the implementation of the proposed proof systems on top of  \texttt{Coq} \cite{Levesque96-WhatPlanning} or Tableaux \cite{tabu}, and try to find applications in knowledge representation and reasoning.
%

\bibliographystyle{plain} 
\bibliography{mm2011}

\begin{thebibliography}{10}

\bibitem{apt}
Olderog Ernst-R\"{u}diger Apt Krzysztof~R., de Boer Frank~S.
\newblock {\em Verification of Sequential and Concurrent Programs}.
\newblock Springer, third edition, 2009.

\bibitem{Baldoni:2001:RCA:646293.687243}
Matteo Baldoni, Laura Giordano, Alberto Martelli, and Viviana Patti.
\newblock {\em Reasoning about Complex Actions with Incomplete Knowledge: A
  Modal Approach}, volume LNCS 2202 of {\em ICTCS '01}.
\newblock Springer-Verlag, London, UK, UK, 2001.

\bibitem{Baral:1993:RCA:1624140.1624145}
Chitta Baral and Michael Gelfond.
\newblock Representing concurrent actions in extended logic programming.
\newblock In {\em Proceedings of the 13th international joint conference on
  Artifical intelligence - Volume 2}, pages 866--871, San Francisco, CA, USA,
  1993. Morgan Kaufmann Publishers Inc.

\bibitem{cvt}
Chitta Baral, Vladik Kreinovich, and Ra\'{u}l Trejo.
\newblock Computational complexity of planning and approximate planning in the
  presence of incompleteness.
\newblock {\em Artificial Intelligence}, 122:241--267, September 2000.

\bibitem{Cadoli:1997:SKC:1216075.1216081}
Marco Cadoli and Francesco~M. Donini.
\newblock A survey on knowledge compilation.
\newblock {\em AI Commun.}, 10:137--150, December 1997.

\bibitem{Etzioni92anapproach}
Oren Etzioni, Steve Hanks, Daniel Weld, Denise Draper, Neal Lesh, and Mike
  Williamson.
\newblock An approach to planning with incomplete information.
\newblock In {\em In Proc. 3rd Int. Conf. on Principles of Knowledge
  Representation and Reasoning}, pages 115--125. Morgan Kaufmann, 1992.

\bibitem{Gelfond93representingaction}
Michael Gelfond and Vladimir Lifschitz.
\newblock Representing action and change by logic programs.
\newblock {\em Journal of Logic Programming}, 17:301--322, 1993.

\bibitem{tabu}
Reiner H\"{a}hnle.
\newblock Tableaux and related methods.
\newblock Handbook of Automated Reasoning, 2001.

\bibitem{DBLP:conf/ijcai/Kartha93}
G.~Neelakantan Kartha.
\newblock Soundness and completeness theorems for three formalizations of
  action.
\newblock {\em IJCAI93}, pages 724--731, 1993.

\bibitem{Levesque96-WhatPlanning}
Hector~J. Levesque.
\newblock What is planning in the presence of sensing?
\newblock In {\em Proceedings of the National Conference on Artificial
  Intelligence (AAAI)}, pages 1139--1146, Portland, Oregon, 1996. American
  Association for Artificial Intelligence.

\bibitem{sc07lin}
Fangzhen Lin.
\newblock {\em Situation Calculus}, chapter~16, pages 649--669.
\newblock Elsevier, 2007.

\bibitem{Lin97howto}
Fangzhen Lin and Ray Reiter.
\newblock How to progress a database.
\newblock {\em Artificial Intelligence}, 92:131--167, 1997.

\bibitem{Lin:1995:PCT:201019.201021}
Fangzhen Lin and Yoav Shoham.
\newblock Provably correct theories of action.
\newblock {\em J. ACM}, 42:293--320, March 1995.

\bibitem{2005safeplanningqbf}
Marco~De Luca, Enrico Giunchiglia, Massimo Narizzano, and Armando Tacchella.
\newblock "safe planning" as a qbf evaluation problem.
\newblock In {\em Proceedings of the Second RoboCare Workshop}, 2005.

\bibitem{Moore85}
R.~C. Moore.
\newblock A formal theory of knowledge and action.
\newblock In J.~R. Hobbs and R.~C. Moore, editors, {\em Formal Theories of the
  Commonsense World}, pages 319--358, Norwood, NJ, 1985. Ablex.

\bibitem{DBLP:conf/lpnmr/NieuwenborghEV07}
Davy~Van Nieuwenborgh, Thomas Eiter, and Dirk Vermeir.
\newblock Conditional planning with external functions.
\newblock {\em Lecture Notes in Computer Science}, 4483:214--227, 2007.

\bibitem{oglietti05incomplete}
Marcelo Oglietti.
\newblock Understanding planning with incomplete information and sensing.
\newblock {\em Artificial Intelligence}, 164(1-2):171--208, May 2005.

\bibitem{Otwell04aneffective}
Charles Otwell, Anja Remshagen, and Klaus Truemper.
\newblock An effective qbf solver for planning problems.
\newblock In {\em MSV/AMCS, CSREA Press}, pages 311--316, 2004.

\bibitem{Petrick04extendingthe}
Ronald P.~A. Petrick and Fahiem Bacchus.
\newblock Extending the knowledge-based approach to planning with incomplete
  information and sensing.
\newblock In {\em In ICAPS-04}, pages 2--11. AAAI Press, 2004.

\bibitem{satHandbookplanning}
Jussi Rintanen.
\newblock {\em Planning and SAT}, chapter~15, pages 483--503.
\newblock Chapter 15, Handbook of Satisfiability, IOS Press, 2009.

\bibitem{ccama}
Boaz~Barak Sanjeev~Arora.
\newblock Computational complexity:a modern approach.
\newblock {\em Cambridge University Press}, 2009.

\bibitem{DBLP:conf/aaai/ScherlL93}
Richard~B. Scherl and Hector~J. Levesque.
\newblock The frame problem and knowledge-producing actions.
\newblock In {\em AAAI}, pages 689--695, 1993.

\bibitem{Scherl:2003:KAF:767243.767244}
Richard~B. Scherl and Hector~J. Levesque.
\newblock Knowledge, action, and the frame problem.
\newblock {\em Artif. Intell.}, 144:1--39, March 2003.

\bibitem{sb}
Tran~C. Son and Chitta Baral.
\newblock Formalizing sensing actions: A transition function based approach.
\newblock {\em Artificial Intelligence}, 125(1-2):19--91, 2001.

\bibitem{citeulike:4295998}
Phan~H. Tu, Tran~C. Son, and Chitta Baral.
\newblock Reasoning and planning with sensing actions, incomplete information,
  and static causal laws using answer set programming.
\newblock {\em Theory Pract. Log. Program.}, 7(4):377--450, 2007.

\end{thebibliography}

\end{document}